\documentclass[letterpaper, 10 pt, conference]{ieeeconf}
\IEEEoverridecommandlockouts

\usepackage{cite}
\usepackage{amsmath,amssymb,amsfonts}
\usepackage{algorithmic}
\usepackage{graphicx}
\usepackage{subcaption} 
\usepackage{microtype}
\usepackage{textcomp}
\usepackage{xcolor}
\usepackage{float}
\usepackage{multirow}
\usepackage{url}
\def\BibTeX{{\rm B\kern-.05em{\sc i\kern-.025em b}\kern-.08em
    T\kern-.1667em\lower.7ex\hbox{E}\kern-.125emX}}

\DeclareMathAlphabet\mathbfcal{OMS}{cmsy}{b}{n}

\newtheorem{theorem}{Theorem}
\newtheorem{assumption}{Assumption}

\begin{document}

\title{Mind the Gaps: Multi-Robot Feedback-Driven Ergodic Coverage in Unknown Environments}
\author{Thales C. Silva \& Nora Ayanian%
\thanks{This work was supported by the
Office of Naval Research (ONR) Award No. N00014-24-1-2662 and the Brown University Seed Award from the Office of the Vice President for Research.
        The authors are with the Department of Computer Science at Brown University, Providence-RI, USA. 
        {\tt\small \{thales\textunderscore silva,nora\textunderscore ayanian\}@brown.edu}}%
}
        
\maketitle

\begin{abstract} 
In this work, we address the problem of multi-robot adaptive coverage, where teams of robots perform dynamic sampling by continuously adjusting their positions to collect data in an environment. This task can be challenging, particularly when robots must be efficiently allocated to new sampling locations over time.
Ergodic search methods optimize robot trajectories by ensuring that the robots’ time-averaged spatial distribution aligns with the spatial distribution of environmental information. While these methods promote effective exploration provided a target distribution, they often fail to account for unknown prior distributions of the environment. To overcome this limitation, we propose an adaptive coverage strategy that utilizes real-time feedback from an environmental model to adjust robot sampling behavior in response to unknown conditions.
Our approach enhances traditional ergodic trajectory optimization by constructing a target spatial information distribution based on parametric models of the environment, which are updated online. This strategy assumes that the environment is either static or changes slowly compared to the robot's motion. Our framework allows robots to dynamically prioritize regions of high interest, improving coverage efficiency, synthesizing effective control policies for individual agents, and optimizing resource use in settings with unknown prior distributions.
We validate our approach through simulations, demonstrating its effectiveness in enhancing coverage and resource allocation.
\end{abstract}


\section{Introduction}
Multi-agent systems involve multiple agents interacting, coordinating, and competing to perform complex tasks~\cite{Li2019}. These systems have practical application in various fields such as robotic swarms~\cite{Yu2021}, environmental monitoring ~\cite{Lynch2008,sung2023decision}, and surveillance~\cite{Yu2020}. In tasks that include data-acquisition, such as search and rescue \cite{Heintzman2021}, characterization of ocean currents \cite{ehrich2016cooperative}, multi-agent exploration \cite{Alysson2024},
the proper coverage planning poses a particular set of 
challenges, including {building a model of the dynamic environment, quantifying the model uncertainty, and assigning the robots to new sampling points according to a uncertainty quantification method.} 
This allows the multi-robot team to adequately consider the spatial distribution of environmental information while optimizing trajectories \cite{Dressel2019,dong2023time,garg2018persistent}. By addressing these issues, robots can persistently monitor a spatiotemporal
environment. In this paper, we
consider these problems and propose a solution for the characterization and coverage of a spatial distribution of interest.

In general, adaptive sampling strategies can be segmented in three distinct parts, each with its own degree of complexity: i) a model of the process to be sampled, ii) an uncertain quantification metric, and iii) a planning approach that will lead the agents to desired sampling locations.
In this work, we leverage the planning flexibility of ergodic search methods \cite{Miller2016,naveed2024eclares,dong2023time,Mathew2011}, which allow us to tackle the issue of optimizing search trajectories while properly considering spatial distributions.  {In particular, we do not focus on modifying the underlying traditional ergodic search algorithm \cite{mathew2010uniform,Dressel2019}. Instead, we focus on systematically constructing the target distribution using estimates of the environmental process and the robot’s sensing model.}

\begin{figure}
\centering
     \includegraphics[width =0.8\columnwidth]{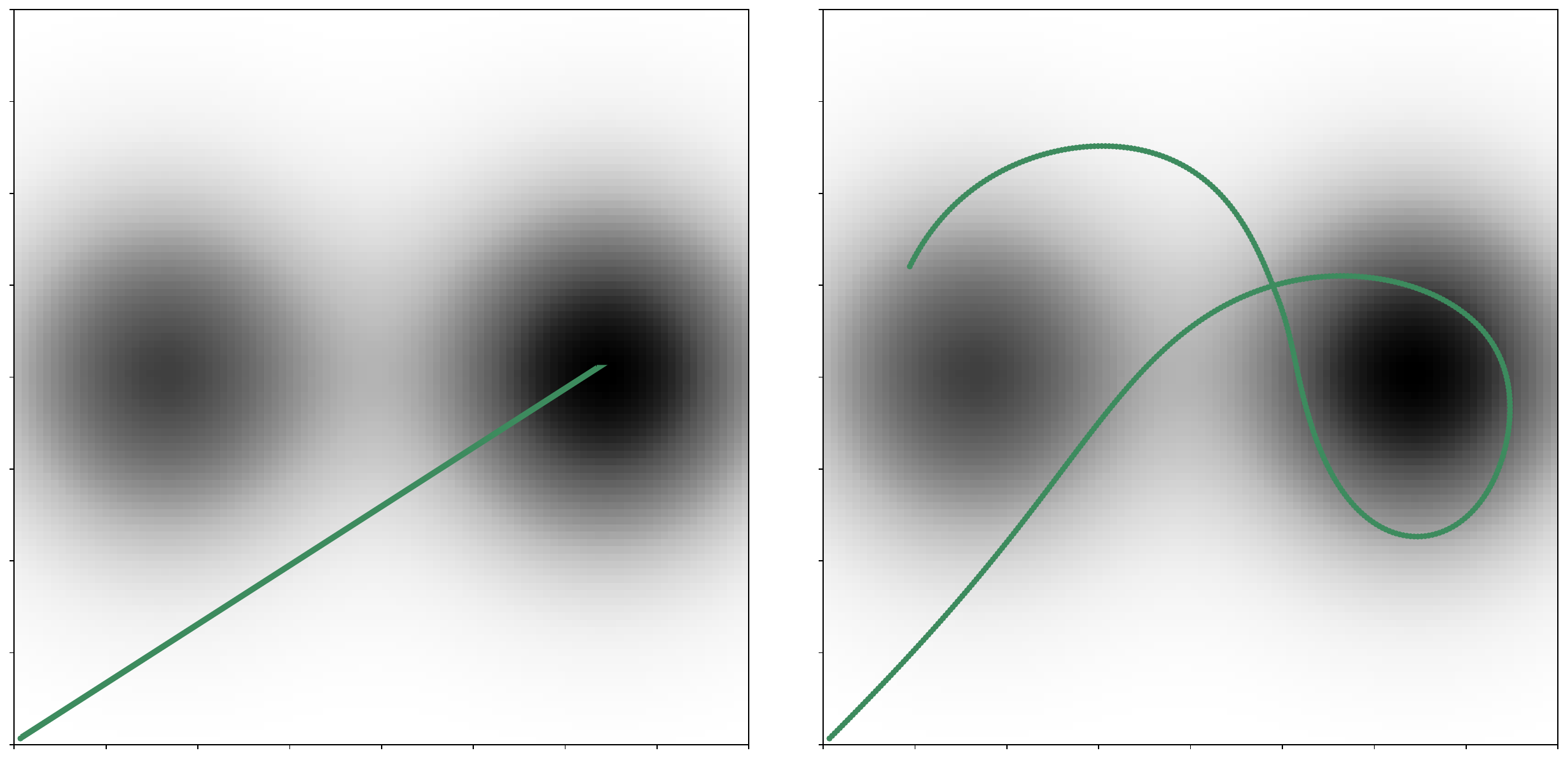}
    \caption{\small
    Example of ergodic trajectory (right) and a trajectory that moves to the highest density point (left). Both trajectories start from identical initial conditions.}
    \label{fig:search_path}
\end{figure}

An ergodic trajectory finds a balance between exploitation and exploration. It spends more time sampling areas that provide valuable information while still covering a wider range instead of being limited to the highest-density points \cite{Miller2016,Dressel2019}. Figure \ref{fig:search_path} illustrates the trajectory of such a coverage strategy where the density given by a bimodal Gaussian distribution. This strategy increases resiliency to potential modeling inaccuracies by encouraging a thorough and adaptable sensing approach.
There are many methods to compute an ergodic trajectory. In \cite{Sylvain2023,Mezic2017}, a heat equation is used to consider the difference between the goal and coverage distributions of a multi-robot system, guiding the agents' movement through a potential field. While Lee \textit{et al.} \cite{lee2024stein} proposes a Stein variational inference method to handle distribution of trajectories in continuous time and space. However, projection-based trajectory optimisation \cite{Dressel2019} is, traditionally, the most widespread method.
  
{While these works focus on \textit{trajectory planning}, the representation of the distribution of the environment is key to closing the loop on adaptive control strategies.
For example, in \cite{Patel2021,Dressel2019,Mathew2011} the desired distribution is pre-defined by the user.
Expanding on that, \cite{Miller2016} uses a Bayesian filter to represent the belief state of the desired distribution and updates it over time. 
In Naveed \textit{et al.} \cite{naveed2024eclares}, information decay is combined with a traditional ergodic search algorithm to model the lack of measurements over time. In addition, they also incorporate an energy-aware filter to ensure the robots have enough battery during the mission.}
In contrast, our approach addresses this challenge by employing a parametric adaptive model to represent the informational distribution in the environment. The robot team collects data samples in dynamic settings, where the ability to track a changing sensory function is enhanced by using a forgetting factor in the data weighting process. 

By incorporating adaptive modeling strategies into a ergodic search, we gain the ability to systematically address a core challenge of adaptive coverage. We capture the underlying distribution and changes in information density. This enables robots to continuously update their understanding of the environment process in real-time. 
This planning approach ensures that the robots do not ignore regions with lower density, while allocating visitation time proportionally to the target distribution. By not modifying the core ergodic algorithm but instead adapting the model used as target distribution, we balance efficient exploration with targeted sampling in areas of high interest.
In a nutshell, the contributions of this work are: 
\begin{itemize}
    \item [I)] An online adaptive method to update the target distribution of an ergodic search algorithm for a multi-robot team, which allows for coverage of unknown distributions without priors.
    \item [II)] The seamless incorporation of target models with an ergodic method for multi-robot system. We show that, for a static environment, our adaptive model converges to the process of interest.
\end{itemize}

The remainder of the paper is organized as follows: the robot team model and problem formulation are described in Section \ref{sec:problem}. Our feedback control law is introduced in Section \ref{sec:feedback_control}. 
Results are in Section \ref{sec:results}, and a discussion is in Section \ref{sec:discussion}. 
Finally, conclusions are in Section \ref{sec:conclusion}. 

\section{Problem Set-Up}
\label{sec:problem}
In this section we formulate the coverage problem. We assume the robots have integrator dynamics,
\begin{align}
    \dot x_i = u_i,
    \label{eq:first_order_dyn}
\end{align}
for $i=1,...,N$, where $x_i\in \mathbb{R}^2$ and $u_i\in\mathbb{R}^2$ are the robot's position and input, respectively. In general, to achieve such a dynamics we can assume that a low-level controller can be in place to enforce (\ref{eq:first_order_dyn}).
We assume the robots carry sensors to sample a dynamical process of interest in the environment and that they have sufficient computational power to keep a local representation of the environment. We consider the existence of a function that maps environment locations to the information of interest, \textit{i.e.,} $\phi:A \mapsto \mathbb{R}$, where $A\subset\mathbb{R}^2$ is a region in the environment. Such a function acts as a measure of importance across $A\subset \mathbb{R}^2$, guiding robot paths. We aim to plan paths in which robots spends more time where $\phi$ is high and less time where it is low. While the function $\phi$ is not directly known to the robots, their sensors can measure its value at their respective locations $x_i$.
The definition of such a function is quite general and fits different applications. For example, in a scenario where a team of robots is deployed to monitor a oil spill, the sensory function could represent the oil density in the environment. In a pollution tracking application using air quality sensors, $\phi$ might be defined as the concentration of harmful particulates at different locations.
{We consider a multi-robot system performing persistent coverage of a environment with a spatiotemporal process. Specifically, we study the following problem.}

{\textbf{Problem 1:} 
Given a team of robots equipped with sensors, we seek to optimize their collective coverage of an environment $A\subset \mathbb{R}^2$ containing a process of interest $\phi$. The environment is characterized by a spatial importance function unknown to the robots but can be estimated through local sensor readings at the robot’s positions.
Design trajectories that ensure persistent sampling of the environment, maintaining updated estimates of $\phi$ based on the robots’ most recent measurements, and maximize their time in regions with high values of $\phi$ while not neglecting regions with low values.
}

The solution of Problem 1 must account for the inherent lack of knowledge of the distribution of the regions of interest in the environment, aiming to maximize the system's overall sampling efficiency.

\subsection{Adaptive Ergodic Search}
In this section, we state the definitions of ergodic search that are useful to this work.
{An ergodic system is one where the time spent in a particular state matches the probability of finding it in that state at any random moment \cite{Cornishbowden2020}.}
We want a strategy that leads the agents to spend a proportional amount of time on regions in which the distribution of the underlying function of interest is higher, without neglecting regions where there is less information to be gained. To attain such a behavior we leverage the strategy presented in \cite{Mathew2011,Dressel2019}. Instead of modifying the algorithm from \cite{Mathew2011,Dressel2019}, this paper centers on formulating a strategy for systematically constructing a local representation of the distribution of interest using the environment and sensing models. 

Let the target distribution be the estimate of the underlying distribution of interest, $\phi$, with evolution described by
\begin{align}
    z(t_f)=T(z(t_0),t_0,t_f),
    \label{eq:target_dist}
\end{align}
where $z(t) \in A \subset \mathbb{R}^2$, in which $A$ is a bounded region, $T$ is the corresponding map that describes the evolution of the target position, and $t_0$ and $t_f$ are the initial and final time, respectively. Let the counterimage of a set $U\subseteq A$ under the transformation $T(\cdot, t_0, t_f)$ be defined as
\begin{align}
    T^{-1}(\cdot, t_0, t_f)(U) = \{x: T(x, t_0, t_f)\in U\},
\end{align}
which allow us to define the evolution of the density function of the target distribution $z(t_f)$. Given the initial uncertainty of the target distribution $\mu(0,x)=\mu_0(x)$, then its evolution can be described by (see \cite{lasota2013chaos} for a detailed discussion),
\begin{align}
    \mu(t,U) = \int_{T^{-1}(\cdot, t_0, t_f)(U)}\mu(0,x)dx,
\end{align}
where $\mu(t,U)$ denotes the probability measure of the set $U$. We employ the Frobenius-Perron operator $P$ corresponding to the transformation $T(\cdot, t_0, t_f)$ to denote the relationship between  $\mu(t,U)$ and $\mu(0,U)$, such that
\begin{align}
    \int_{U}P^{[t_0,t]}\mu(0,x)dx = \int_{U}\mu(t,x)dx.
\end{align}

Inspired by  \cite{mathew2010uniform}, consider the spherical set $B(x,r)$ with radius $r$ and center at $x$, and define the following tube set 
\begin{align*}
    H^t(B(x,r))=\{(y,\tau):\tau \in [0,t] \text{ and } T(y, \tau, t_f){\in} B(x,r)\}.
\end{align*}
In words, this tube set is the set of all points of the target trajectory inside spherical set $B(x,r)$, when evolved according to the target dynamics \eqref{eq:target_dist}. 
Given the sensor dynamics in (\ref{eq:first_order_dyn}), the 
fraction of the total time $t$
that the sensor trajectory spends inside the tube set $H^t(B(x,r))$ is 
\begin{align}
    d^t(x,r) = \frac{1}{Nt}\sum_{i=1}^N \int_{t_0}^t \boldsymbol{1}_{B(x,r)}(T(x_i(\tau), t_0, \tau))d\tau,
    \label{eq:avg_time_tube}
\end{align}
where $\boldsymbol{1}_{B(x,r)}(\cdot)$ is the indicator function of the set $B(x,r)$,
\textcolor{black}{\textit{i.e.,} it equals $1$ when the trajectory is inside 
$B(x,r)$ and $0$ otherwise}. As pointed out in \cite{mathew2010uniform}, we can compute (\ref{eq:avg_time_tube}) through the following spherical integrals,
\begin{align}
    d^t(x,r) = \int_{B(x,r)} C^t(y) dy,
\end{align}
with
\begin{align}
    C^t(y) = \frac{1}{Nt} \sum_{i=1}^N \int_0^t P^{[\tau,t]}\delta_{x_i(\tau)}(y)d\tau,
    \label{eq:coverage_distribution}
\end{align}
where $\delta_{a}$ is the delta distribution with mass at the point $a$. 
Equations \eqref{eq:avg_time_tube} (\ref{eq:coverage_distribution}) are dimensionless and gives the proportion of time the agents spend inside the tube set $H^t(B(x,r))$.

For certain applications (\textit{e.g.}, uniform sampling), it is important that the proportion of time the sensor trajectories spend within the tube set remains close to distribution of the target dynamics. This can be captured by the metric
\begin{align}
    \label{eq:target_density}
    E^2(t) = \int_0^R\int_A\big(C^t(B(x,r))-\mu(t,B(x,r))\big)^2dxdr, 
\end{align}
 which compares the integrals of the distributions, where $R\in\mathbb{R}$ is largest radius in which the sphere regions are considered, \textit{i.e.,} the largest neighborhoods used to compare distributions, and $A\subset \mathbb{R}^2$ is the bounded environment.

It was shown in \cite{Mathew2011,Dressel2019}, that (\ref{eq:target_density}) is equivalent to a metric induced by a Sobolev space of norm negative,
\begin{align}
    \varphi^2(t)=\lVert C^t-\mu(t,\cdot)\lVert ^2 _{H^{-(s+1)/2}}=\sum_{k\in\mathbb{Z}^n}\Lambda_k\lvert s_k(t) \lvert^2,
    \label{eq:sobolev_metric}
\end{align}
where $\lVert \cdot \lVert_{H^{-(s+1)/2}}$ represents the Sobolev space with negative index, $s_k(t)=c_k(t)-\mu_k(t)$, $\Lambda_k=\frac{1}{(1+\lVert k \lVert ^2)^\ell}$,  $\ell=(n+1)/2$, $c_k(t)=\langle C^t, f_k \rangle$, and $\mu_k(t)=\langle\mu(t,\cdot),f_k\rangle$ where $\{f_k\}$ are the Fourier basis functions that satisfy the Neumann boundary conditions and $k$ is the corresponding wave-number vector.

Intuitively, minimizing \eqref{eq:sobolev_metric} (and consequently \eqref{eq:target_density}) ensures that the fraction of time the robot trajectories spend within 
$B(x,r)$ at multiple spatial scales (considered through $R$) matches the target distribution, \textit{i.e.,} drives the system to ergodicity. Consequently, robots spend more time in high-density regions while still visiting lower-density areas.
\subsection{Environment Representation}

We assume that the function
that maps environment locations to the information of interest, $\phi(x)$, varies slowly compared with the dynamics of the robots, which is in accordance with the principle of timescale separation\footnote{From the perspective of the slower dynamics, the fast subsystem can often be approximated as being in a quasi-steady state.}  \cite{khalil2002nonlinear}. 

We use a basis function approximation scheme for each robot to estimate the function $\phi(x)$. In particular, the sensor function approximation for the $i$th robot is given by
\begin{align}
    \label{eq:env_functionapp}
    \hat \phi_i(x)=\mathcal{K}(x)^T\hat a_i,
\end{align}
in which $\hat a_i$ is a parameter vector that must be tuned online, and $\mathcal{K}:\mathbb{R}^2\mapsto \mathbb{R}^m$ is a vector of bounded continuous basis functions. Regarding this representation, we require that the following assumption holds,

\begin{assumption}
    \label{assumption_parameters}
    There is an ideal parameter vector $a\in\mathbb{R}^m$ unknown to the robots, such that 
    \begin{align}
        \phi(x)=\mathcal{K}(x)^Ta.
        \label{eq:env_function}
    \end{align}
\end{assumption}
    
In theory, over a bounded domain, a function can be approximated arbitrarily well by a set of basis functions \cite{Slotine1991}. However, designing a suitable set of basis functions requires application-specific expertise. We can approximate functions using various basis functions such as polynomial, sigmoids, and piecewise-linear functions. In our case, we use Gaussian basis functions due to their smoothness, and localized influence induced by exponential decay. 

Hence, we let the density of the target distribution be defined by
\begin{align}
    \label{eq:target_distribution_def}
    \mu(x)=\frac{\hat{\phi}(x)}{\int_A \hat{\phi}(x)dx}.
\end{align}
Consequently, \eqref{eq:target_distribution_def} defines the target density using the current estimate of the information map, $\hat{\phi}(x)$. We next propose an adaptive update law which, assuming sufficiently rich trajectories, converges to the true information function. To promote such richness, a feedback controller is chosen to minimize the ergodic metric in \eqref{eq:target_density}. The resulting framework operates in a closed loop: measurements update the estimate $\hat{\phi}$, the target distribution $\mu(x)$ is recomputed, and the ergodic controller generates trajectories that drive further exploration and improve the estimate. Our framework is illustrated in Figure \ref{fig:pipeline}.

\begin{figure*}[h]
     \centering
     \includegraphics[width =\textwidth]{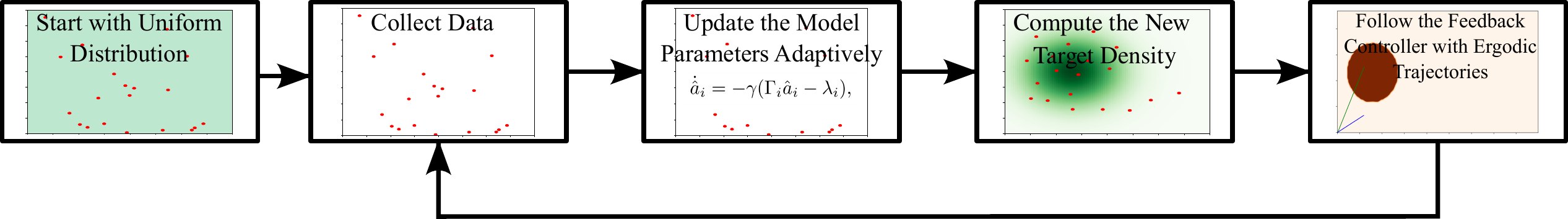}
     \caption{
     \small
     We present the environment model as feedback for the ergodic search algorithm used in dynamic environments. We adaptively update the model parameters according to equations (\ref{eq:adaptive_Gamma}) and (\ref{eq:adaptve_lambda}). The robots gather data iteratively based on new distributions, update the model, and calculate new target densities.
     }
     \label{fig:pipeline}
 \end{figure*}

 \section{Feedback Adaptive Control Law}
 \label{sec:feedback_control}

 We want a strategy to drive the multi-robot team to configurations in which the coverage maximizes the collection of meaningful data while not completely disregarding regions with lower probability density.
 We emphasize that such a task is non-trivial since the robots do not have knowledge of the density distributions a priori. To overcome this challenge, we propose a strategy that models the density distribution online, while we use ergodic algorithm to explore and cover regions of interest.

 Based on \eqref{eq:sobolev_metric}, let us define the following function,
\vspace{-0.15cm}
 \begin{align}
    \label{eq:cost_function}
     \Phi(t) = \frac{1}{2}N^2t^2\varphi^2(t) = \frac{1}{2}\sum_k \Lambda_k\lvert S_k(t) \lvert^2,
 \end{align}
 with $S_k(t)=Nts_k(t)$. The cost function is defined as the first time-derivative of (\ref{eq:cost_function}) at the end of the horizon $t+\Delta t$,
 \begin{align*}
     C(t,\Delta t) = \dot \Phi(t+\Delta t) = \sum_k \Lambda_k S_k(t+\Delta t)\dot S_k(t+\Delta t).
 \end{align*}
 Then, the feedback law for each agent is computed by taking the limit of $\Delta t$ to zero. (Please, see \cite{mathew2010uniform,Mathew2011,Dressel2019} for details on this derivation):
\begin{align}
    \label{eq:feedback_law}
    & u_i(t) = -u_{\text{max}}\frac{B_i}{\lVert B_i \lVert_2},
\end{align}
where
$
    B_i = \sum_k \Lambda_k S_k(t) \nabla f_k(x_i(t)).
$
Using (\ref{eq:feedback_law}), we establish a feedback control law 
 designed to minimize the spherical integrals of the difference of robot trajectory distributions and the estimate of the environmental process density \eqref{eq:target_density}. 
 
 The model parameters $\hat a_i$ are adjusted according to an adaptation law inspired by \cite{Slotine1991}. To proceed, we first consider the following quantities,
\begin{align}
\label{eq:adaptive_Gamma}
   & \Gamma_i(t) = \int_0^t w(\tau)\mathcal{K}_i(\tau)\mathcal{K}_i(\tau)^Td\tau \text{ and}\\
\label{eq:adaptve_lambda}
   & \lambda_i(t) = \int_0^t w(\tau)\mathcal{K}_i(\tau)\phi_i(\tau)d\tau, 
\end{align}
which can be computed differentially by robot $i$ with zero initial conditions. We use the shorthand notation $\mathcal{K}_i(t)= \mathcal{K}(x_i(t))$ for the value of the basis function at the position of the robot $i$.  As pointed out in \cite{Slotine1991}, the function $w(t)\geq 0$ defines the data collection weighting, which can be used to emphasize the importance of recent data over old data. This approach improves the coverage of dynamical processes. Finally, the adaption law for $\hat a_i$ is now defined as
\begin{align}
    \label{eq:parameter_adaptation}
    \dot{\hat{a}}_i = -\alpha(\Gamma_i \hat{a}_i-\lambda_i),
\end{align}
where $\alpha \in \mathbb{R}_{>0}$ is a positive constant that can be seen as learning rate. Such an adaption law is effectively a gradient descent on the error of the sensory information over time. We highlight that, due to the structure of the cost function in (\ref{eq:cost_function}) this strategy can be implemented in a centralized or all-to-all communication network. Extension for a decentralized implementation of this approach is left for future research, {and further discussion is given in the Conclusion section.}

 The behavior of our estimation for adaptive coverage is formalized in the following theorem.

 \begin{theorem}
 \label{theorem:1}
  Let the multi-robot system with first-order dynamics (\ref{eq:first_order_dyn}) in closed-loop with the feedback control law (\ref{eq:feedback_law}). Under the adaptation law in (\ref{eq:parameter_adaptation}) and {Assumption \ref{assumption_parameters}, the robots' environmental model converges to an approximation of the sensory function over the set $\Omega_i=\{x_i(\tau):\tau \geq 0, w(\tau) > 0\}$, composed by the points on the robots' trajectory with positive weights.}
 \end{theorem}
 \begin{proof}
    Consider the following Lyapunov candidate function
    \begin{align}
        V = \frac{1}{2} \sum_{i=1}^n\tilde{a}_i^T \tilde{a}_i,
    \end{align}
    \textcolor{black}{where $\tilde{a}_i=\hat{a}_i-a_i$, and $a>0$.} Note that such a function is quadratic in the parameter errors $\tilde{a}_i$, thus $V$ is bounded below by zero and is positive definite. Taking its time derivative along the trajectories of the adaptation law, noticing that under the assumption of timescale separation (\textit{i.e.,} that $a(t)$ varies sufficiently slowly), $\dot{\tilde{a}}_i\approx \dot{\hat{a}}_i$, we have
    \begin{align}
        \dot V = \sum_{i=1}^n \tilde{a}_i^T  \dot{\hat{a}}_i,
    \end{align}
    substituting $\dot{\hat{a}}_i$ with the adaptation law in (\ref{eq:parameter_adaptation}) yields
    \begin{align}
        \dot V = -\alpha \sum_{i=1}^n \tilde{a}_i^T (\Gamma_i \hat{a}_i-\lambda_i).
    \end{align}
    Now, considering the quantities in (\ref{eq:adaptive_Gamma}) and (\ref{eq:adaptve_lambda}), and recalling Assumption 1,
    \begin{IEEEeqnarray}{rlr}
        \dot V =& -\alpha \sum_{i=1}^n \tilde{a}_i^T 
        \Big(\int_0^t w(\tau)\mathcal{K}_i(\tau)\mathcal{K}_i(\tau)^T\hat{a}_id\tau 
        \nonumber
        \\
        &-\int_0^t w(\tau)\mathcal{K}_i(\tau)\phi_i(\tau)d\tau\Big)
        \nonumber \\
        = & -\alpha \sum_{i=1}^n \tilde{a}_i^T 
        \int_0^t w(\tau)\mathcal{K}_i(\tau)\mathcal{K}_i(\tau)^T\tilde{a}_id\tau,
\end{IEEEeqnarray}
which is negative definite for $\int_0^t w(\tau)\mathcal{K}_i(\tau)\mathcal{K}_i(\tau)^Td\tau>0$. Therefore, we have that $V>0$ and $\dot V<0$, for $\Omega_i=\{x_i(\tau):\tau \geq 0, w(\tau) > 0\}$. Consequently, for sufficiently slowly varying $a(t)$, the robots model asymptotically converges to an approximation of the sensory function.
 \end{proof}
 Notice that Theorem \ref{theorem:1} gives asymptotic convergence of the robot's local model to the environmental sensory function, $\phi_i$. If $\lim_{t\rightarrow\infty} \Gamma_i(t)>0$, the strategy will be sufficiently rich, allowing the robots' model to ultimately converge to the desired model \cite{Schwager2008}. In our approach, the updated model of the target distribution is used in the ergodic cost metric \eqref{eq:target_density}, which enhances its expressiveness and seamlessly integrate with the adaptive strategy.

 If the Assumption 1 does not hold--such as when parameters shift too rapidly for the controller to track accurately--one might question the strategy's effectiveness. This is a practical concern, as real-world parameters often change dynamically. 
 While a comprehensive robustness analysis is object of a future paper, we notice that similar analysis for adaptive controllers can be found in the literature, including \cite{Slotine1991} and several foundational texts on adaptive control (\textit{e.g.}, \cite{Kaiwen2021,WANG2017,WeiLin2002,TSINIAS2000}). Numerical results indicate that the adaptation law adjusts the parameters to achieve small tracking error of the desired model. Performance of tracking is measured by the integral of the squared error, as outlined in Section \ref{sec:results}.

\section{Results}
\label{sec:results}
In this section, we present simulations of our proposed approach and compare with a baseline where there is no adaptive update of the target distribution. In the simulations, $4$ agents are considered with the integrator model in (\ref{eq:first_order_dyn}) with a maximum speed of 1.0 $m/s$ and sampling frequency of 2.0 $Hz$. Our simulation code is available at 
\url{https://github.com/scthales/mind-the-gaps-ergodic}.

\subsection{Static Gaussians}
\label{ex:bimodal_1}
In the first example, we consider a target distribution given by a bimodal static Gaussian distribution. In Figure~\ref{fig:trajectories} we show the trajectories of our approach in (c) and of an ergodic search with a fixed uniform target distribution in (b). The underlying model of interest is shown in Figure \ref{fig:static_gaussian} (a). Figure \ref{fig:static_gaussian} (d) shows the evolution of the normalized RMSE between the approximated model and the underlying function of interest over the entire environment of interest $A\subset \mathbb{R}^2$.

\begin{figure*}[h]
    \centering
    \begin{subfigure}[t]{0.2\textwidth}
        \centering
        \includegraphics[width=\textwidth]{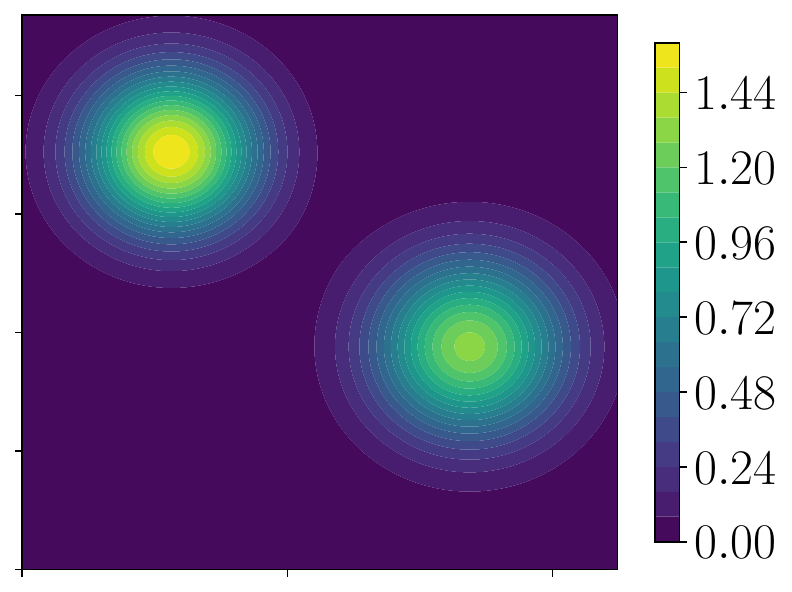}
        \caption{Underlying  distribution. 
        }
    \end{subfigure}
    \begin{subfigure}[t]{0.2\textwidth}
        \centering
        \includegraphics[width=\textwidth]{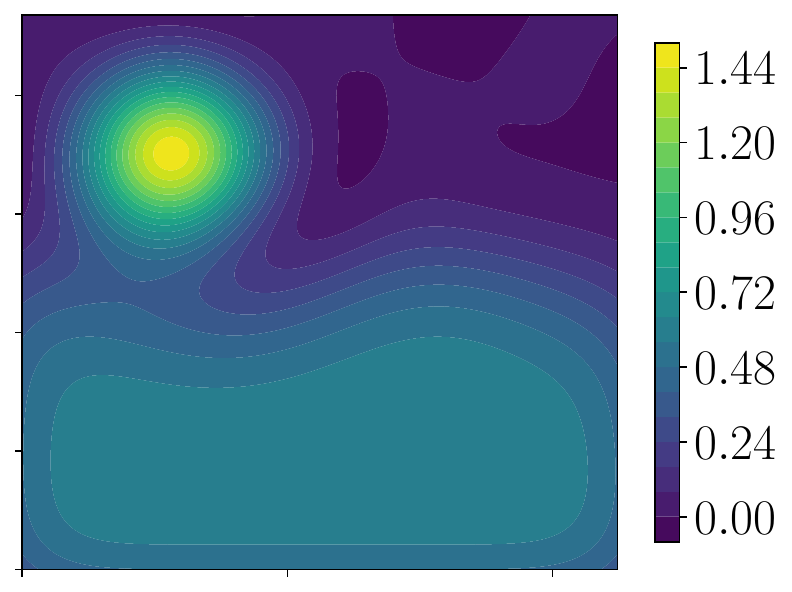}
        \caption{Fixed uniform target\\ distribution $\mu(x)$}
    \end{subfigure} 
    \begin{subfigure}[t]{0.2\textwidth}
        \centering
        \includegraphics[width=\textwidth]{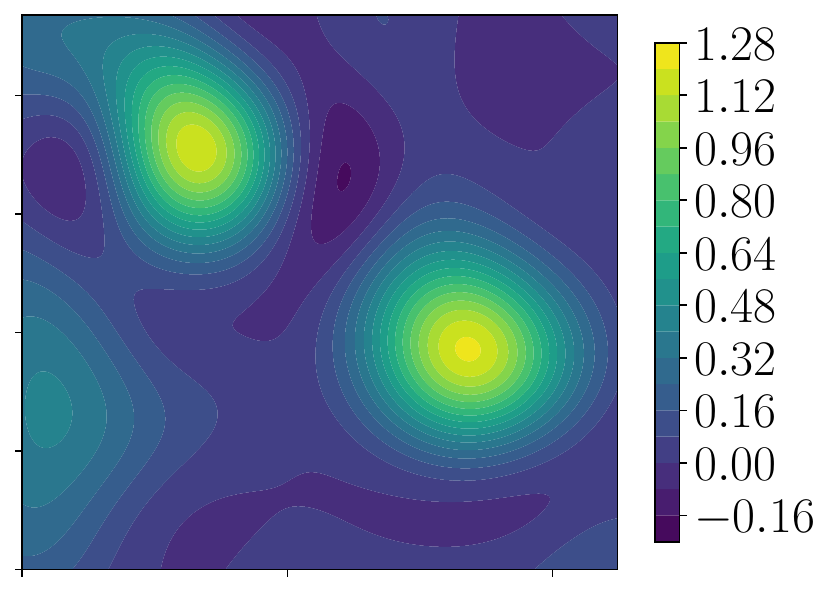}
        \caption{Adaptive target distribution $\mu(x)$ (our approach) 
        }
    \end{subfigure}
    \vspace{-0.15cm}
    \begin{subfigure}[t]{0.2\textwidth}
        \centering
        \includegraphics[width=\textwidth]{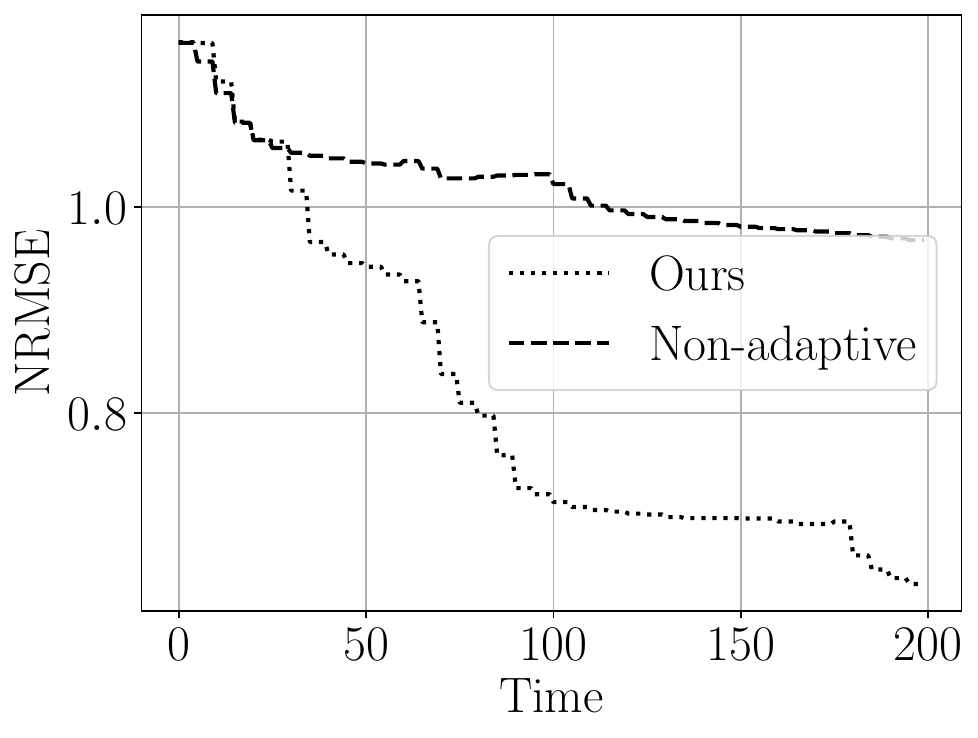}
        \caption{RMSE 
        }
    \end{subfigure}
    \caption{\small
    In (a) is the underlying distribution, given by two Gaussians. In (b) and (c) is shown
    the agents model using the local function approximation, as in \eqref{eq:env_functionapp}, of the environment at the end of the run.
    In (d) is the root mean square error (RMSE) between each of the approaches and the underlying distribution. 
    } 
    \normalsize
    \label{fig:static_gaussian}
\end{figure*}

\begin{figure}[h]
    \centering
    \begin{subfigure}[t]{0.22\textwidth}
        \centering
        \includegraphics[width=\textwidth]{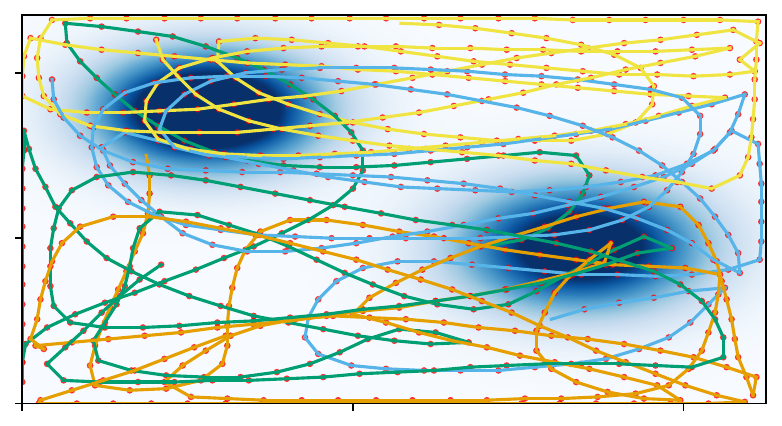}
        \caption{Uniform target\\ distribution 
        }
    \end{subfigure}
    \begin{subfigure}[t]{0.22\textwidth}
        \centering
        \includegraphics[width=\textwidth]{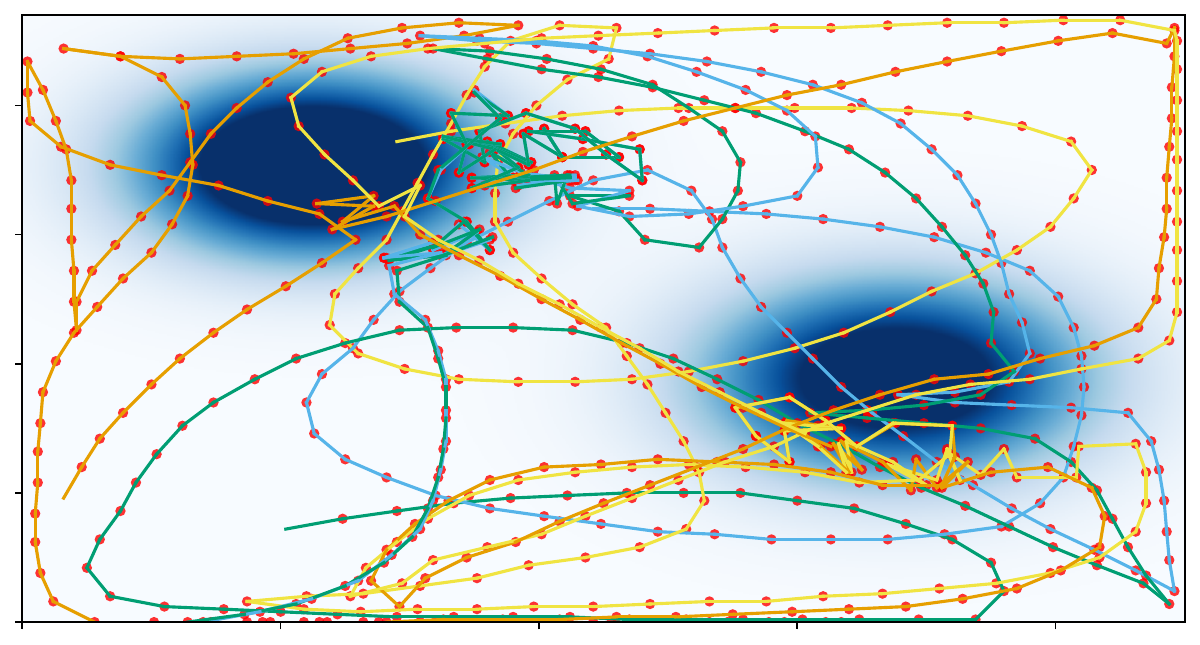}
        \caption{Adaptive target\\ distribution}
    \end{subfigure} 
    \caption{\small
    Trajectories for ergodic search with uniform (a) and with adaptive (b) target distributions. Red dots along the trajectories represent sampling locations.
    } 
    \normalsize
    \label{fig:trajectories}
\end{figure}

\subsection{Time-Varying Gaussians}
In this example, we focus on a time-varying environment. We use the same conditions as in Section IV-A, but we let the Gaussians move in opposite directions. To set their speed, we use a symmetric s-curve parametrized by $v_{g} = s_{sim}^\gamma / (s_{sim}^\gamma + (1-s_{sim})^\gamma)$, where $s_{sim}$ is the normalized simulation step and $\gamma\in(0,1)$. This gives a profile for the the speed of the target, with higher speed at the beginning and end, and lower in the middle. For values of $\gamma$ closer to 1 the speed is higher while for lower values the maximum speed is lower.
The results are shown in Figure \ref{fig:moving_rmse} for $\gamma=0.2$ and $\gamma=0.3$, where we show the RMSE for different time steps and for two different speeds of the target distribution. 


\begin{figure}[h]
    \centering
    \begin{subfigure}[t]{0.22\textwidth}
        \centering
        \includegraphics[width=\textwidth]{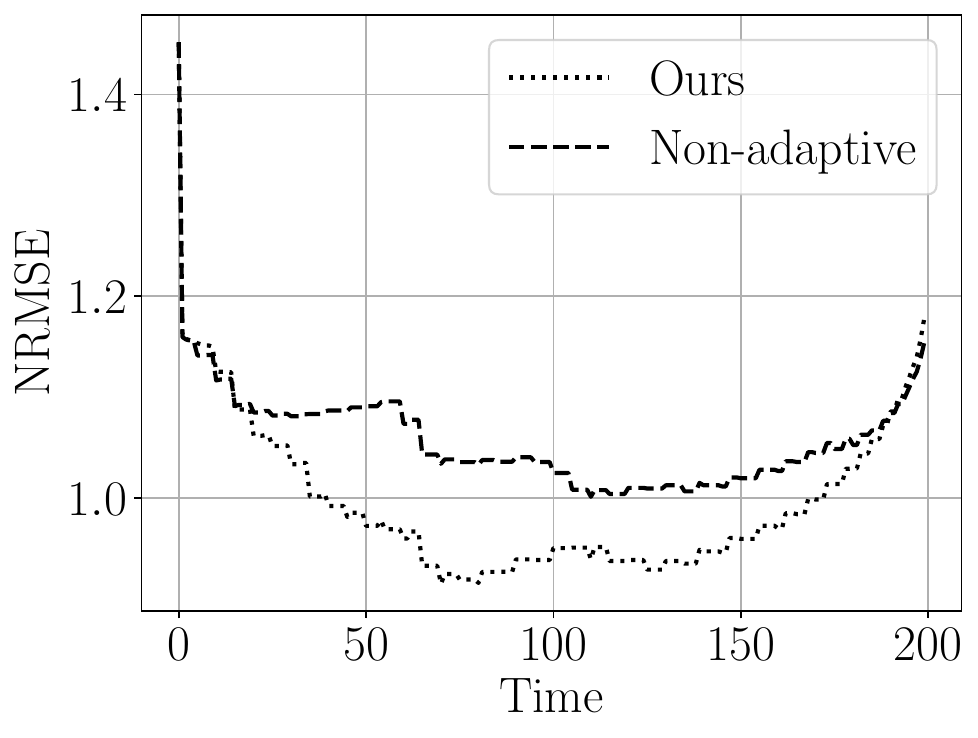}
        \caption{Moving target\\ $\gamma=0.2$ 
        }
    \end{subfigure}
    \begin{subfigure}[t]{0.22\textwidth}
        \centering
        \includegraphics[width=\textwidth]{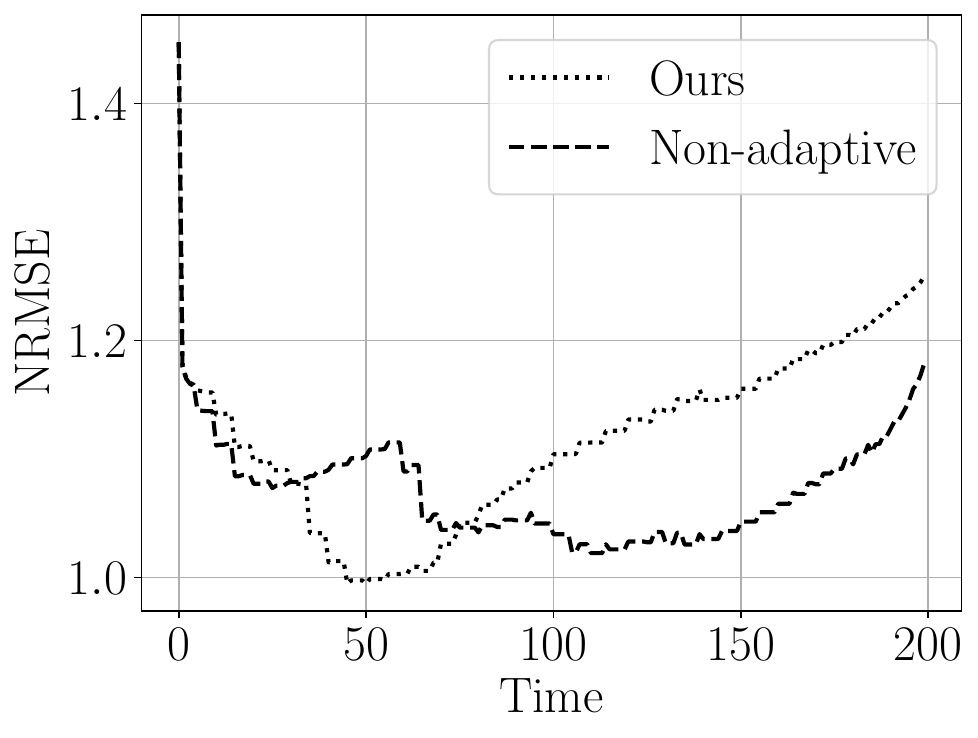}
        \caption{Moving target\\ $\gamma = 0.3$}
    \end{subfigure} 
    \caption{\small
     \textcolor{black}{Root mean square for moving targets with different speeds. Note that our method works well for lower speeds. While for higher speeds, our method starts to break-down and uniformly sampling the environment leads to better model approximations.}
    } 
    \normalsize
    \label{fig:moving_rmse}
\end{figure}


\section{Discussion}
\label{sec:discussion}

The timescale separation, which allowed us to treat the target dynamics as quasi-steady while analyzing the long-term modeling, is necessary to show convergence of the parameters to the real function. For a dynamic process that varies sufficiently slow, we noticed that such an assumption is sufficient for dynamic coverage when robots do not have a model of the environmental process. 
However, as we showed in the examples, for a process that evolves with dynamics close to the robots' speed, our approach start to break down. We hypothesize that a similar break down will be unavoidable with any strategy that doesn't have  a prior model and demands time to update the process model. In addition, the
desired properties only hold in the limit as
time approaches infinity. Thus, defining the relationship between quality of the model and number of samples is an interesting point for future research.

\section{Conclusion}
\label{sec:conclusion}
In this paper we investigated the problem of adaptive coverage control for multi-robot systems in scenarios in which the agents do not have prior information about the distribution of desired information. 
As a first result, we proposed an online adaptive method to update the target distribution of an ergodic search algorithm for multi-robot agents. Consequently, in our approach, rather than designing a strategy to assign the robots to new sampling locations, we relied on ergodic search to persistently monitor an environment with no prior information about the distribution of desired information. We use the natural exploration quality of ergodic search algorithms to gather initial data, and as the local model improves, focus the search on areas of higher interest without ignoring areas with lower distributions.
Then we show that our incorporation of the adaptive target model and ergodic search asymptotically converges to the sensory function for relatively slow environments.

Finally, we show through numerical simulations the effectiveness of our approach in a time-varying scenario. In addition, we tested our algorithm with varying values of time-varying distributions and illustrated the limits of our approach. There are many possible lines of interest for future work, we plan to formally study the time-varying conditions that allow an adaptive strategy to converge. While we assumed relatively slowly varying distributions, it can be possible to provide explicit bounds on the time-derivative of the parameters. Another line of interest is to derive a completely distributed strategy for the ergodic search, allowing the robots to adaptively and distributively search a dynamic environment.
\vspace{-0.2cm}
\bibliographystyle{IEEEtran}
\bibliography{model_plus_env}

\end{document}